\documentclass[11pt]{article}
\usepackage{fullpage, amsmath, amsthm}
\newif\iffull
\fulltrue
\usepackage{times}
\usepackage{graphicx} % more modern
\usepackage{subfig}

\usepackage{natbib}
\usepackage{balance}
\usepackage{algorithm}
\usepackage{algorithmic}
\usepackage{color}
\usepackage{wrapfig}

\usepackage{hyperref}

\newcommand{\para}[1]{\smallskip \noindent {\bf #1.}}

\newcommand{\bal}{\mathrm{balance}}
\newcommand{\red}{\mbox{\sc red}}
\newcommand{\blue}{\mbox{\sc blue}}

\title{Fair Clustering  Through Fairlets}
\author{
Flavio Chierichetti \\
Dipartimento di Informatica \\
Sapienza University \\
\and
Ravi Kumar $\quad$ Silvio Lattanzi $\quad$ Sergei Vassilvitskii\\
Google Research \\
}

\usepackage{amsmath, amssymb, amsthm}
\usepackage{nicefrac}

\newtheorem{theorem}{Theorem}
\newtheorem{definition}[theorem]{Definition}
\newtheorem{lemma}[theorem]{Lemma}

\newcommand{\Y}{\mathcal{Y}}
\newcommand{\CC}{\mathcal{C}}

\begin{document}

\title{Fair Clustering  Through Fairlets\thanks{This work first appeared at NIPS 2017}}
\author{
Flavio Chierichetti$^1$
%Dipartimento di Informatica \\
%Sapienza University 
%Rome, Italy
\and
Ravi Kumar$^2$ \and Silvio Lattanzi$^2$ \and Sergei Vassilvitskii$^2$
%Google Research \\
%1600 Amphitheater Parkway \\
%Mountain View, CA 94043 
%\and 
%Silvio Lattanzi \\
%Google Research \\
%76 9th Ave \\
%New York, NY 10011  
%\and
%Sergei Vassilvitskii \\
%Google Research \\
%76 9th Ave \\
%New York, NY 10011
}
\date{
$^1$ Dipartimento di Informatica, Sapienza University \\
$^2$ Google Research 
}

\maketitle 

\begin{abstract} 

  We study the question of fair clustering under the {\em disparate
    impact} doctrine, where each protected class must have
  approximately equal representation in every cluster. We formulate
  the fair clustering problem under both the $k$-center and the
  $k$-median objectives, and show that even with two protected classes
  the problem is challenging, as the optimum solution can violate
  common conventions---for instance a point may no longer be assigned
  to its nearest cluster center!

  En route we introduce the concept of \emph{fairlets}, which are
  minimal sets that satisfy fair representation while approximately
  preserving the clustering objective.  We show that any fair
  clustering problem can be decomposed into first finding good
  fairlets, and then using existing machinery for traditional
  clustering algorithms.  While finding good fairlets can be NP-hard,
  we proceed to obtain efficient approximation algorithms based on
  minimum cost flow.

  We empirically 
  quantify the value of fair clustering on real-world datasets with
  sensitive attributes.
\end{abstract} 

%%%

\section{Introduction}
\label{sec:intro}

From self driving cars, to smart thermostats, and digital assistants, machine learning is behind many of the technologies we use and rely on every day.  Machine learning is also increasingly used to aid with decision making---in awarding home loans or in sentencing recommendations in courts of law~\citep{Sentencing}. While the learning algorithms are not inherently biased, or unfair, the algorithms may pick up and amplify biases already present in the training data that is available to them. Thus a recent line of work has emerged on designing \emph{fair} algorithms.

The first challenge is to formally define the concept of fairness, and indeed recent work shows that some natural conditions for fairness cannot be simultaneously achieved~\citep{KleinbergImpossibility, GoelImpossibility}. In our work we follow the notion of {\em disparate impact} as articulated by~\cite{Feldman}, following the {\em Griggs v. Duke Power Co.} US Supreme Court case. Informally, the doctrine codifies the notion that protected attributes, such as race and gender, should not be {\em explicitly} used in making decisions, {\em and the decisions made} should not be disproportionately different for applicants in different protected classes. In other words, if an unprotected feature, for example, height, is closely correlated with a protected feature, such as gender, then decisions made based on height may still be unfair, as they can be used to effectively discriminate based on gender.

%\begin{wrapfigure}{rt}{0.4\textwidth}
\begin{figure}
\centering
\includegraphics[width=0.38\textwidth,keepaspectratio]{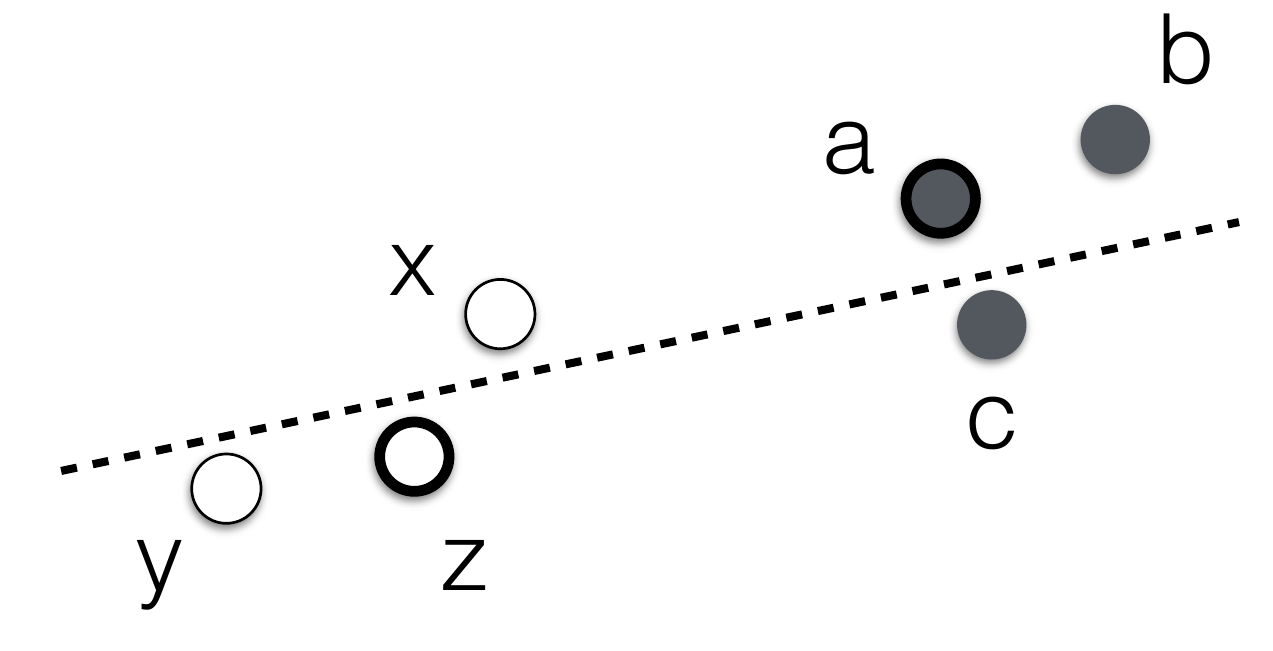}
\caption{\small A colorblind $k$-center clustering algorithm would group points $a,b,c$ into one cluster, and $x,y,z$ into a second cluster, with centers at $a$ and $z$ respectively. A fair clustering algorithm, on the other hand, may give a partition indicated by the dashed line. Observe that in this case a point is no longer assigned to its nearest cluster center. For example $x$ is assigned to the same cluster as $a$ even though $z$ is closer. }
\label{fig:example}
%\end{wrapfigure}
%\vspace{-0.3in}
\end{figure}

While much of the previous work deals with supervised learning, here we consider the most common unsupervised learning problem: clustering. In modern machine learning systems, clustering is often used for feature engineering, for instance augmenting each example in the dataset with the id of the cluster it belongs to in an effort to bring expressive power to simple learning methods. In this way we want to make sure that the features that are generated are fair themselves. As in standard clustering literature, we are given a set $X$ of points lying in some metric space, and our goal is to find a partition of $X$ into $k$ different clusters, optimizing a particular objective function.  We assume that the coordinates of each point $x \in X$ are unprotected; however each point also has a color, which identifies its protected class. The notion of disparate impact and fair representation then translates to that of color balance in each cluster.  We study the two color case, where each point is either {\em red} or {\em blue}, and show that even this simple version has a lot of underlying complexity. We formalize these views and define a fair clustering objective that incorporates both fair representation and the traditional clustering cost; see Section~\ref{sec:prelim} for exact definitions.

A clustering algorithm that is \emph{colorblind}, and thus does not take a protected attribute into its decision making, may still result in very unfair clusterings; see Figure~\ref{fig:example}. This means that we must explicitly use the protected attribute to find a fair solution. Moreover, this implies that a fair clustering solution could be strictly worse (with respect to an objective function) than a colorblind solution. %The ratio of the objective function value between fair and colorblind solutions define an explicit \emph{cost of fairness}.

Finally, the example in Figure~\ref{fig:example} also shows the main technical hurdle in looking for fair clusterings. Unlike the classical formulation where every point is assigned to its nearest cluster center, this may no longer be the case. Indeed, a fair clustering is defined not just by the position of the centers, but also by an {\em assignment function} that assigns a cluster label to each input. 

\para{Our contributions}
In this work we show how to reduce the problem of fair clustering to that of classical clustering via a pre-processing step that ensures that any resulting solution will be fair. In this way, our approach is similar to that of~\citet{Zemel}, although we formulate the first step as an explicit combinatorial problem, and show approximation guarantees that translate to approximation guarantees on the optimal solution. Specifically we: 
\begin{itemize}
\item[(i)] Define fair variants of classical clustering problems such as $k$-center and $k$-median; 
\item[(ii)] Define the concepts of fairlets and fairlet decompositions, which encapsulate minimal fair sets; 
\item[(iii)] Show that any fair clustering problem can be reduced to first finding a fairlet decomposition, and then using the classical (not necessarily fair) clustering algorithm; 
\item[(iv)] Develop approximation algorithms for finding fair decompositions for a large range of fairness values, and complement these results with NP-hardness; and 
\item[(v)] Empirically quantify the price of fairness, i.e., the ratio of the cost of traditional clustering to the cost of fair clustering.
\end{itemize}

\para{Related work}
Data clustering is a classic problem in unsupervised learning that
takes on many forms, from partition clustering, to soft clustering,
hierarchical clustering, spectral clustering, among many others. See,
for example, the books by~\cite{ClusteringBook1, ClusteringBook2} for
an extensive list of problems and algorithms. In this work, we focus
our attention on the $k$-center and $k$-median problems. Both of these
problems are NP-hard but have known efficient approximation
algorithms. The state of the art approaches give a 2-approximation for
$k$-center~\citep{Gonzalez85} and a $(1+
\sqrt{3}+\epsilon)$-approximation for $k$-median~\citep{LiSvensson13}.

Unlike clustering, the exploration of fairness in machine learning is
relatively nascent. There are two broad lines of work. The first is in
codifying what it means for an algorithm to be fair. See for example
the work on statistical parity~\citep{Luong, Kamishima}, disparate
impact~\citep{Feldman}, and individual
fairness~\citep{FairnessAwareness}. More recent work
by~\citet{GoelImpossibility} and~\citet{KleinbergImpossibility} also
shows that some of the desired properties of fairness may be
incompatible with each other.

A second line of work takes a specific notion of fairness and looks
for algorithms that achieve fair outcomes. Here the focus has largely
been on supervised learning~\citep{Luong, Hardt} and
online~\citep{Morgenstern} learning. The direction that is most
similar to our work is that of learning intermediate representations
that are guaranteed to be fair, see for example the work
by~\citet{Zemel} and~\citet{Kamishima}. However, unlike their work, we
give strong guarantees on the relationship between the quality of the
fairlet representation, and the quality of any fair clustering
solution.

In this paper we use the notion of fairness known as {\em disparate
impact} and introduced by~\citet{Feldman}. This notion is also closely
related to the $p\%$-rule as a measure for fairness. The $p\%$-rule is
a generalization of the $80\%$-rule advocated by US Equal Employment
Opportunity Commission~\citep{biddle2006adverse} and was used in a
recent paper on mechanism for fair
classification~\citep{Gummadi17}. In particular our paper addresses an
open question of~\citet{Gummadi17} presenting a framework to solve an
unsupervised learning task respecting the $p\%$-rule.
% Finally our notion is also related to other measure of fairness considered
%in previous works~\citep{DBLP:journals/kais/KamiranC11, Kamishima, Zemel}.

\section{Preliminaries}
\label{sec:prelim}

Let $X$ be a set of points in a metric space equipped with a distance
function $d: X^2 \rightarrow \mathbb{R}^{\geq 0}$.  For an integer
$k$, let $[k]$ denote the set $\{ 1, \ldots, k \}$.  

We first recall standard concepts in clustering.  A \emph{$k$-clustering}
$\CC$ is a partition of $X$ into $k$ disjoint subsets, $C_1, \ldots,
C_k$, called \emph{clusters}.  We can evaluate the
quality of a clustering $\CC$ with different objective functions.  In
the \emph{$k$-center} problem, the goal is to minimize
$$
\phi(X, \CC) = \max_{C \in \CC} \min_{c \in C} \max_{x \in C} d(x, c),
$$
and in the \emph{$k$-median} problem, the goal is to minimize
$$
\psi(X, \CC) = \sum_{C \in \CC} \min_{c \in C} \sum_{x \in C} d(x, c).
$$
A clustering $\CC$ can be equivalently described via an {\em
  assignment} function $\alpha : X \rightarrow [k]$.  The points in
cluster $C_i$ are simply the pre-image of $i$ under $\alpha$, i.e.,
$C_i = \{ x \in X ~\mid~ \alpha(x) = i\}$.

Throughout this paper we assume that each point in $X$ is colored
either red or blue; let $\chi: X \rightarrow \{ \red, \blue \}$ denote
the color of a point.  For a subset $Y \subseteq X$ and for $c \in \{ \red,
\blue\}$, let $c(Y) = \{ x \in X ~\mid~ \chi(x) = c \}$ and let $\#c(Y)
= |c(Y)|$.  

We first define a natural notion of balance.
\begin{definition}[Balance]
  For a subset $\varnothing \ne Y \subseteq X$, the $\bal$ of $Y$ is defined as:
$$
\bal(Y) = \min\left( \frac{\#\red(Y)} {\#\blue(Y)},
  \frac{\#\blue(Y)}{\#\red(Y)} \right) \in [0, 1].
$$
The balance of a clustering
$\CC$ is
defined as:
$$
\bal(\CC) = \min_{C \in \CC} \bal(C).
$$
\end{definition}
A subset with an equal number of red and blue points has balance $1$
(perfectly balanced) and a monochromatic subset has balance $0$ (fully
unbalanced).  To gain more intuition about the notion of balance, we
investigate some basic properties that follow from its definition.
\begin{lemma}[Combination]
  Let $Y, Y' \subseteq X$ be disjoint.  If $\CC$ is a clustering of $Y$
  and $\CC'$ is a clustering of $Y'$, then $\bal(\CC \cup \CC')=
  \min(\bal(\CC), \bal(\CC'))$.
\end{lemma}
It is easy to see that for any clustering $\CC$ of $X$, we have
$\bal(\CC) \leq \bal(X)$.  In particular, if $X$ is not perfectly
balanced, then no clustering of $X$ can be perfectly balanced.  We
next show an interesting converse, relating the balance of $X$ to the
balance of a well-chosen clustering.
\begin{lemma}
\label{lem:minimal-fairness}
Let $\bal(X) = b/r$ for some integers $1 \leq b \leq r$ such that
$\gcd(b, r) = 1$.  Then there exists a clustering $\Y = \{Y_1, \ldots,
Y_m\}$ of $X$ such that (i) $|Y_j| \leq b+r$ for each $Y_j \in \Y$,
i.e., each cluster is small, and (ii) $\bal(\Y) = b/r = \bal(X)$.
\end{lemma}
\iffull
\begin{proof}
  Without loss of generality, let $B = \#\blue(X) \leq \#\red(X) = R$.
  By assumption, $B/R = b/r$.  We construct the clustering $\Y$
  iteratively as follows.

  If $(R - B) \geq (r - b)$, then we remove $r$ red points and $b$
  blue points from the current set to form a cluster $Y$.  By
  construction $|Y| = b + r$ and $\bal(Y) = b/r$.  Furthermore the
  leftover set has balance $(B - b)/(R - r) \geq b/r$ and we iterate
  on this leftover set.

  If $(R - B) < (r - b)$, then we remove $(R - B) + b$ red points and
  $b$ blue points from the current set to form $Y$.  Note that $|Y|
  \leq b + r$ and that $\bal(Y) = b/(R - B + b) \geq b/r$.

  Finally note that when the remaining points are such that the red and the
  blue points are in a one-to-one correspondence,  we can pair
  them up into perfectly balanced clusters of size 2.
\end{proof} 
\fi

\paragraph{Fairness and fairlets.}

Balance encapsulates a specific notion of fairness, where a clustering
with a monochromatic cluster (i.e., fully unbalanced) is considered
unfair.  We call the clustering $\Y$ as described in
Lemma~\ref{lem:minimal-fairness} a {\em $(b,r)$-fairlet decomposition}
of $X$ and call each cluster $Y \in \Y$ a \emph{fairlet}. %More formally,

%\begin{definition}[Fairlet decomposition]
%\label{def:fd}
%A $(b,r)$-fairlet decomposition is a clustering $\Y$ of the point in $X$
%such that for any cluster $Y\subseteq \Y$ we have $\#\red(Y)=r$ and
%$\#\blue(Y)=b$.
%\end{definition}
%
%It is easy to see that fairlet
%decompositions are not unique and using
%Lemma~\ref{lem:minimal-fairness}, one can refine any clustering into
%another clustering with the same balance where all the clusters have
%smaller size.

Equipped with the notion of balance, we now revisit the clustering
objectives defined earlier.  The objectives do not consider the color
of the points, so they can lead to solutions with monochromatic
clusters.  We now extend them to incorporate fairness.
\begin{definition}[$(t, k)$-fair clustering problems]
\label{def:fk}
In the \emph{$(t, k)$-fair center} (resp., \emph{$(t, k)$-fair
  median}) problem, the goal is to partition $X$ into $\mathcal{C}$
such that $|\CC|=k$, $\bal(\mathcal{C}) \geq t$, and $\phi(X, \CC)$
(resp. $\psi(X, \CC)$) is minimized.
\end{definition}
Traditional formulations of $k$-center and $k$-median eschew the
notion of an assignment function. Instead it is implicit through a set
$\{c_1, \ldots, c_k\}$ of centers, where each point assigned to its
nearest center, i.e., $\alpha(x) = \arg\min_{i \in [1,k]} d(x, c_i).$
Without fairness as an issue, they are equivalent formulations;
however, with fairness, we need an explicit assignment function (see
Figure~\ref{fig:example}).

\iffalse
Consider for
example, $X$ in $\mathbb{R}^1$ that has red points at $r_1 = -3, r_2 =
-1$ and blue points at $b_1 = 1$ and $b_2 = 2$. With $k = 2$, an
unbalanced solution has cost $2$, putting $r_1, r_2$ in one cluster
and $b_1, b_2$ in another. A more balanced solution, on the other hand
has a cost of $4$ and puts $r_1, b_1$ in one cluster and $r_2, b_2$ in
another. Observe that in this solution a point is no longer assigned to the 
cluster center nearest to it. \footnote{We can modify this example to increase the gap
  between the costs of the unbalanced and the balanced solutions
  arbitrarily.}
\fi

\iffull
\else
Missing proofs are deferred to the full version of the paper. 
\fi

\section{Fairlet decomposition and fair clustering}

At first glance, the fair version of a clustering problem appears
harder than its vanilla counterpart. In this section we prove
 a reduction from the former to the latter.  We do this
by first clustering the original points into small clusters preserving
the balance, and then applying vanilla clustering on these smaller
clusters instead of on the original points.

As noted earlier, there are different ways to partition the input to
obtain a fairlet decomposition.  We will show next that the choice of
the partition directly impacts the approximation guarantees of the
final clustering algorithm.

Before proving our reduction we need to introduce some additional
notation.  Let $\Y = \{ Y_1, \ldots, Y_m \}$ be a fairlet
decomposition. For each cluster $Y_j$, we designate an arbitrary point
$y_j \in Y_j$ as its \emph{center}.  Then for a point $x$, we let
$\beta: X \rightarrow [1,m]$ denote the index of the fairlet to which
it is mapped. We are now ready to define the cost of a fairlet decomposition

\begin{definition}[Fairlet decomposition cost]
\label{def:cfd}
For a fairlet decomposition, we define its \emph{$k$-median cost} as
$\sum_{x\in X}d(x,\beta(x)), $ and its \emph{$k$-center cost} as
$\max_{x \in X} d(x, \beta(x)).$ We say that a $(b,r)$-fairlet
decomposition is \emph{optimal} if it has minimum cost among all
$(b,r)$-fairlet decompositions.
\end{definition}

Since $(X, d)$ is a metric, we have from the triangle
inequality that for any other point $c \in X$,
$$
d(x, c) \leq d(x, y_{\beta(x)}) + d(y_{\beta(x)}, c).
$$

Now suppose that we aim to obtain a $(t, k)$-fair clustering of the
original points $X$.  (As we observed earlier, necessarily $t \leq
\bal(X)$.)  To solve the problem we can cluster instead the centers of
each fairlet, i.e., the set $\{y_1, \ldots, y_m\} = Y$, into $k$
clusters.  In this way we obtain a set of centers $\{c_1, \ldots,
c_k\}$ and an assignment function $\alpha_Y : Y \rightarrow [k]$.

We can then define the overall assignment function as $\alpha(x) =
\alpha_Y(y_{\beta(x)})$ and denote the clustering induced by $\alpha$
as $\CC_\alpha$.  From the definition of $\Y$ and the property of
fairlets and balance, we get that $\bal(\CC_\alpha) = t$.
We now need to bound its cost. Let $\tilde{Y}$ be a multiset, where
each $y_i$ appears $|Y_i|$ number of times.
\begin{lemma}
  $\psi(X, \CC_\alpha) = \psi(X, \Y) + \psi(\tilde{Y}, \CC_\alpha)$ and
  $\phi(X, \CC_\alpha) = \phi(X, \Y)  + \phi(\tilde{Y}, \CC_\alpha)$.
\end{lemma}
\iffull
\begin{proof}
  We prove the result for the $k$-median setting; the $k$-center version is
  similar.  Let $\CC_\alpha=\{C_1,\dots,C_k\}$, with corresponding
  centers $\{c_1, \ldots, c_k\}$.  Using the definition of the
  $k$-median objective and the triangle inequality we get,
$$
\psi(X, \CC_\alpha) = \sum_{i  = 1}^k  \sum_{x \in C_i} d(x, c_i) 
\le \sum_{i = 1}^k \sum_{x \in C_i} \left(d(x, y_{\beta(x)}) + d(y_{\beta(x)}, c_i)\right) 
= \psi (X, \Y) + \psi (\tilde{Y},  \CC_\alpha).
\qedhere
$$
\end{proof} 
\fi
%
%\begin{proof}
%Let $\C_\alpha=\{C_1,\dots,C_k\}$. Recall that the cost of $k$-center 
%is $\phi(X, \C) = \max_{i  = 1}^k  \max_{x \in C_i} d(x, c_i)$
%so using the triangle inequality we get:
%\begin{align*}
%\phi(X, \C_\alpha) &= \max_{i  = 1}^k  \max_{x \in C_i} d(x, c_i) \\
%&\le \max_{i = 1}^k \max_{x \in C_i}(d(x, y_{\beta(x)}) + d(y_{\beta(x)}, c_i)) \\
%&= \phi (X, \Y) + \phi(Y, \C_\alpha).
%\end{align*} 
%\end{proof}
%
%The above lemma shows that we can decompose the cost of fair clustering into a cost that
%depends on the fairlet decomposition and the cost of the clustering solution induced by $\alpha$. 
%
%A similar analysis holds for the $k$-median objective. %In particular let $\psi^*$ denote 
%the cost of the optimal (unfair) solution. 

Therefore in both cases we can reduce the fair clustering problem
to the problem of finding a good fairlet decomposition and then
solving the vanilla clustering problem on the centers of the fairlets.
We refer to $\psi(X, \Y)$ and $\phi(X, \Y)$ as the $k$-median and
$k$-center costs of the fairlet decomposition.

%To
%distinguish between different decomposition we introduce the notion of cost of a
%fairlet decomposition for the two problems.
%
%\begin{definition}[Cost of a fairlet decomposition]\label{def:cf}
%Let $\Y = \{Y_1, \ldots, Y_m\}$ be a fairlet decomposition for the points $X=R\cup B$.
%We define the cost of the fairlet for the $k$-center problem(resp. $k$-median problem),
%$cost(\Y)$, as $\phi (X, \Y)$ (resp. $\psi(X, \Y)$).
%\end{definition}
%

% !TEX root = nips_paper.tex
\section{Algorithms}
\label{sec:algo}

In the previous section we presented a reduction from the fair
clustering problem to the regular counterpart. In this section
we use it to design efficient algorithms for fair clustering.

We first focus on the $k$-center objective and show in
Section~\ref{sec:kmedian} how to adapt the reasoning to solve the
$k$-median objective.  We begin with the most natural case in which we
require the clusters to be perfectly balanced, and give efficient
algorithms for the $(1, k)$-fair center problem. Then we analyze the
more challenging $(t, k)$-fair center problem for $t < 1$.  Let $B =
\blue(X), R = \red(X)$.

\subsection{Fair $k$-center warmup: $(1, 1)$-fairlets}

Suppose $\bal(X) = 1$, i.e., ($|R| = |B|)$ and we wish to find a
perfectly balanced clustering. We now show how we can obtain it using
a good $(1,1)$-fairlet decomposition.

\begin{lemma}
  An optimal $(1,1)$-fairlet decomposition for $k$-center
  can be found in polynomial time.
\end{lemma}
\begin{proof}
  To find the best decomposition, we first relate this question to a
  graph covering problem.  Consider a bipartite graph $G = (B \cup R,
  E)$ where we create an edge $E = (b_i, r_j)$ with weight $w_{ij} =
  d(r_i, b_j)$ between any bichromatic pair of nodes. In this case a
  decomposition into fairlets corresponds to some perfect matching in
  the graph. Each edge in the matching represents a fairlet,
  $Y_i$. Let $\Y = \{ Y_i \}$ be the set of edges in the matching.

  Observe that the $k$-center cost $\phi(X, \Y)$ is exactly the cost
  of the maximum weight edge in the matching, therefore our goal is to
  find a perfect matching that minimizes the weight of the maximum
  edge. This can be done by defining a threshold graph $G_{\tau}$ that
  has the same nodes as $G$ but only those edges of weight at most
  $\tau$. We then look for the minimum $\tau$ where the corresponding
  graph has a perfect matching, which can be done by (binary)
  searching through the $O(n^2)$ values.

  Finally, for each fairlet (edge) $Y_i$ we can arbitrarily
  set one of the two nodes as the center, $y_i$.
%
%
%Observe that the cost $\psi(X, Y)$ is exactly the cost of the perfect matching, thus the optimum fairlet decomposition corresponds to the  min cost perfect matching on this graph. 
\end{proof}
Since any fair solution to the clustering problem induces a set of
minimal fairlets (as described in Lemma \ref{lem:minimal-fairness}),
the cost of the fairlet decomposition found is at most the cost of the
clustering solution.
\begin{lemma}
  Let $\Y$ be the partition found above, and let $\phi^*_t$ be the
  cost of the optimal $(t, k)$-fair center clustering. Then $\phi(X,
  \Y) \leq \phi^*_t$.
\end{lemma}
%Since the best approximation for $k$-median is  $1+ \sqrt{3}+\epsilon$ ~\citep{LiSvensson13}, we have: 
%
This, combined with the fact that the best approximation algorithm for
$k$-center yields a $2$-approximation \citep{Gonzalez85}, gives us the
following.
\begin{theorem}
\label{thm:one-fair-kc}
The algorithm that first finds fairlets and then clusters them is 
a $3$-approximation  for the $(1, k)$-fair center problem. 
\end{theorem}

%
%\begin{theorem}
%\label{thm:one-fair}
%The algorithm that first finds fairlets and then clusters them is 
%%a $3$ approximation  for the $1$-fair $k$-center problem and 
%a $2+\sqrt{3}+\epsilon$ approximation for  the $1$-fair $k$-median problem. 
%\end{theorem}

\subsection{Fair $k$-center: $(1, t')$-fairlets} 

Now, suppose that instead we look for a clustering with balance $t
\lneq 1$.  In this section we assume $t = \nicefrac{1}{t'}$ for some
integer $t' > 1$.  We show how to extend the intuition in the matching
construction above to find approximately optimal $(1, t')$-fairlet
decompositions for integral $t' > 1$.

In this case, we transform the problem into a {\em minimum cost flow}
(MCF) problem.%
\footnote{Given a graph with edges costs and capacities, a source, a
  sink, the goal is to push a given amount of flow from source to
  sink, respecting flow conservation at nodes, capacity constraints on
  the edges, at the least possible cost.}
Let $\tau>0$ be a parameter of the algorithm.  Given the points $B,
R$, and an integer $t'$, we construct a directed graph
$H_{\tau}=(V,E)$. Its node set $V$ is composed of two special nodes
$\beta$ and $\rho$, all of the nodes in $B\cup R$, and $t'$ additional
copies for each node $v\in B\cup R$. More formally, 
$${ V =
  \{ \beta,\rho\} \cup B \cup R \cup \left\{ b_i^j ~\mid~ b_i \in B \mbox{
    and } j \in [t'] \right\} \cup \left\{ r_i^j ~\mid~ r_i \in R \mbox{ and } j
  \in [t'] \right\}.  }
$$

The directed edges of $H_{\tau}$ are as follows: 
\begin{itemize}
\item[(i)] A $(\beta,\rho)$ edge with cost $0$ and capacity $\min(|B|,|R|)$. 
\item[(ii)] A $(\beta,b_i)$ edge for each $b_i \in B$, and an $(r_i,\rho)$
edge for each $r_i \in R$.  All of these edges have cost $0$ and
capacity $t'-1$. 
\item[(iii)] For each $b_i \in B$ and for each $j \in [t']$, a $(b_i, b_i^j)$
edge, and for each $r_i \in R$ and for each $j \in [t']$, an $(r_i,
r_i^j)$ edge.  All of these edges have cost $0$ and capacity $1$. 
\item[(iv)] Finally, for each $b_i \in B, r_j \in R$ and for each $1 \le k,
\ell \le t$, a $(b_i^k, r_j^{\ell})$ edge with capacity $1$.  The cost
of this edge %depends on the problem that we aim to solve. If we are
%trying to create fairlets for the $k$-center problem, the cost 
is $1$  if $d(b_i,r_j) \le \tau$ and $\infty$ otherwise. 
%If we are trying to create fairlets for the $k$-median %problem then the cost of the edge
%and cost $d(b_i,r_j)$.
\end{itemize}

To finish the description of this MCF instance, we have now specify
supply and demand at every node. Each node in $B$ has a supply of $1$,
each node in $R$ has a demand of $1$, $\beta$ has a supply of $|R|$,
and $\rho$ has a demand of $|B|$. Every other node has zero supply and
demand. In Figure~\ref{fig:flow} we show an example of this
construction for a small graph.

\begin{figure}%{t!}{0.56\textwidth}
\includegraphics[width=0.85\textwidth,keepaspectratio]{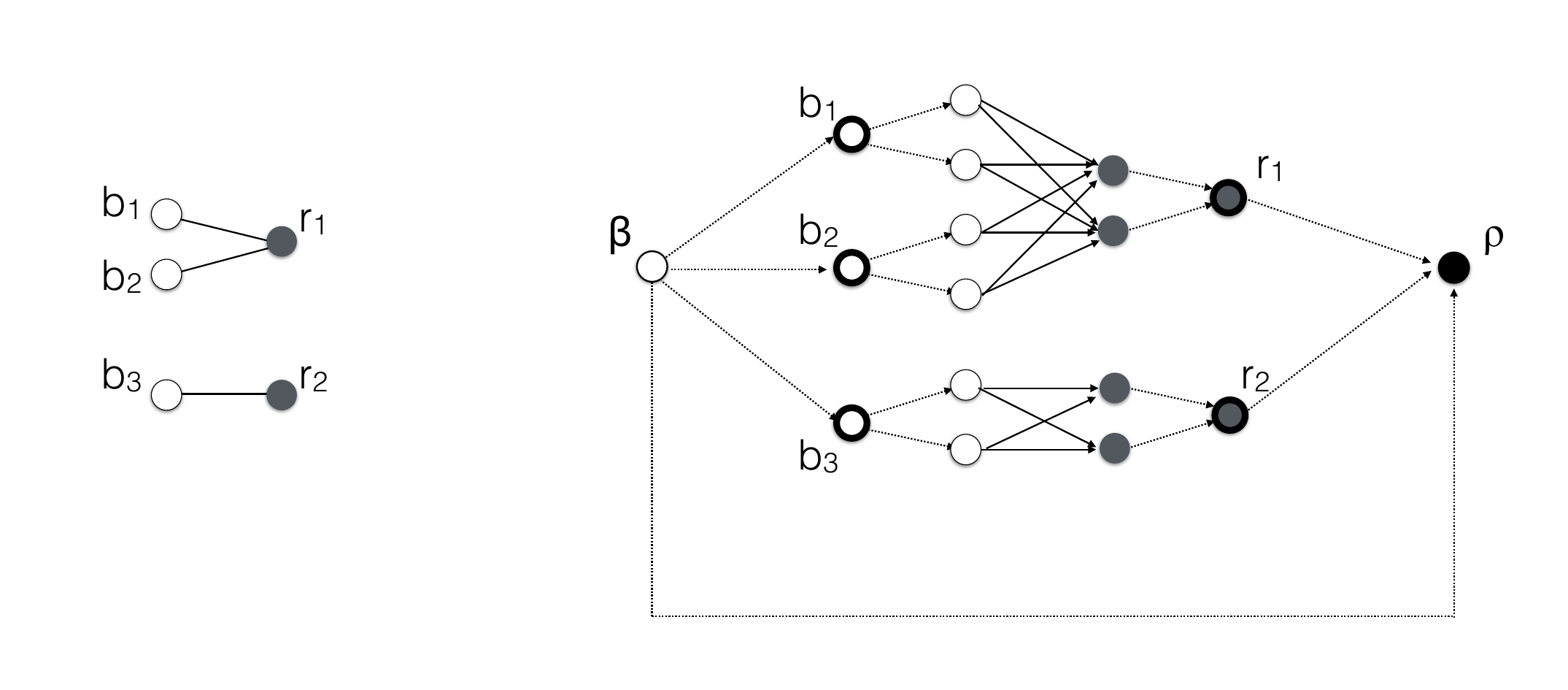}
\caption{\small The construction of the MCF instance for the bipartite graph
  for $t'=2$. Note that the only nodes with positive demands or
  supplies are $\beta, \rho, b_1, b_2, b_3, r_1$, and $r_2$ and all the
  dotted edges have cost $0$.}
\label{fig:flow}
\end{figure}

To finish the description of this MCF instance, we specify the 
supply and demand at every node. Each node in $B$ has a supply of $1$,
each node in $R$ has a demand of $1$, $\beta$ has a supply of $|R|$,
and $\rho$ has a demand of $|B|$. Every other node has zero supply and
demand. In Figure~\ref{fig:flow} we show an example of this
construction for a small graph.

The MCF problem can be solved in polynomial time and since all of the
demands and capacities are integral, there exists an optimal solution
that sends integral flow on each edge. In our case, the solution is a
set of edges of $H_{\tau}$ that have non-zero flow, and the total flow
on the $(\beta,\rho)$ edge.

In the rest of this section we assume for simplicity that any two
distinct elements of the metric are at a positive distance apart and
we show that starting from a solution to the described MCF instance we
can build a low cost $(1,t')$-fairlet decomposition.  We start by
showing that every $(1,t')$-fairlet decomposition can be used to
construct a feasible solution for the MCF instance and then prove that
an optimal solution for the MCF instance can be used to obtain a
$(1,t')$-fairlet decomposition.

\begin{lemma}\label{lem:f2m}
  Let $\Y$ be a $(1,t')$-fairlet decomposition of cost $C$ for the
  $(1/t', k)$-fair center problem. Then it is possible to construct a
  feasible solution of cost $2 C$ to the MCF instance.
\end{lemma}
\begin{proof}
  We begin by building a feasible solution and then bound its cost.
  Consider each fairlet in the $(1,t')$-fairlet decomposition.

  Suppose the fairlet contains $1$ red node and $c$ blue nodes, with
  $c\leq t'$, i.e., the fairlet is of the form
  $\{r_1,b_1,\ldots,b_{c}\}$. For any such fairlet we send a unit of
  flow form each node $b_i$ to $b_i^1$, for $i\in [c]$ and a unit of
  flow from nodes $b_1^1,\ldots,b_{c}^1$ to nodes
  $r_1^1,\ldots,r_1^c$. Furthermore we send a unit of flow from each
  $r_1^1,\ldots,r_1^c$ to $r_1$ and $c-1$ units of flow from $r_1$ to
  $\rho$.  Note that in this way we saturate the demands of all nodes
  in this fairlet.

  Similarly, if the fairlet contains $c$ red nodes and $1$ blue node,
  with $c\leq t'$, i.e., the fairlet is of the form
  $\{r_1,\ldots,r_c,b_1\}$. For any such fairlet, we send $c-1$ units
  of flow from $\beta$ to $b_1$. Then we send a unit of flow from each
  $b_1$ to each $b_1^1,\dots, b_1^c$ and a unit of flow from nodes
  $b_1^1,\ldots,b_{1}^c$ to nodes $r_1^1,\ldots,r_c^1$. Furthermore we
  send a unit of flow from each $r_1^1,\ldots,r_c^1$ to the nodes
  $r_1,\ldots,r_c$.  Note that also in this case we saturate all the
  request of nodes in this fairlet.

  Since every node $v\in B\cup R$ is contained in a fairlet, all of
  the demands of these nodes are satisfied. Hence, the only nodes that
  can have still unsatisfied demand are $\beta$ and $\rho$, but we can
  use the direct edge $(\beta,\rho)$ to route the excess demand, since
  the total demand is equal to the total supply.  In this way we
  obtain a feasible solution for the MCF instance starting from a
  $(1,t')$-fairlet decomposition.

  To bound the cost of the solution note that the only edges with
  positive cost in the constructed solution are the edges between
  nodes $b_i^j$ and $r_k^\ell$. Furthermore an edge is part of the
  solution only if the nodes $b_i$ and $r_k$ are contained in the same
  fairlet $F$. Given that the $k$-center cost for the fairlet
  decomposition is $C$, the cost of the edges between nodes in $F$ in
  the constructed feasible solution for the MCF instance is at most
  $2$ times this distance. The claim follows.
\end{proof}
Now we show that given an optimal solution for the MCF instance of
cost $C$, we can construct a $(1,t')$-fairlet decomposition of cost no
bigger than $C$.
\begin{lemma}\label{lem:m2f}
%Let $\mathcal{S}$ be an optimal solution for the {\em Minimum Cost Flow} problem constructed
%with radius $\tau$ then 
%we can construct a $(1,t)$-fairlet decomposition of cost $C\leq\tau$ for the $t$-fair $k$-center
%problem.
%
  Let $\Y$ be an optimal solution of cost $C$ to the MCF instance.
  Then it is possible to construct a $(1,t')$-fairlet decomposition
  for $(1/t',k)$-fair center problem of cost at most $C$.
\end{lemma}
\iffull
\begin{proof}
  First we show that from an optimal solution for the MCF instance, it
  is possible to construct a $(1,t')$-fairlet decomposition. Then we
  bound the cost of the decomposition.

  Let $E^{\star}$ be the subset of edges of $\{\{b_i,r_j\} ~\mid~ b_i
  \in B \mbox{ and } r_j \in R\}$ such that $b_i$ and $r_j$ are
  connected by edges used in the feasible solution for the MCF.
  Denote by $G_{E^{\star}}$ the graph induced by $E^{\star}$.  Note
  that by construction the degree of each node $b_i$ or $r_j$ in
  $G_{E^{\star}}$ is at most $t'$ and at least $1$.

  We claim that $G_{E^{\star}}$ is a collection of stars, each having
  a number of leaves in $\{1,\ldots, t'\}$. In fact, suppose a
  component of $E^{\star}$ is not a star, that is it contains $x$ red
  nodes, and $y$ blue nodes, with $x,y \ge 2$. In this case there are
  two nodes $r, b \in E^\star$ each with degree at least $2$ that are
  connected to each other. We can safely remove the edge $(r,b)$ while
  still guaranteeing that every red and every blue node in the
  component has at least one neighbor of the opposite color. Removing
  this edge will decrease the cost of the solution which contradicts
  the optimality of $E^{\star}$. (Indeed, the flow that passed through
  the removed edge can be rerouted through the $(\beta,\rho)$ edge.)

  Therefore, we can define the $(1,t')$-fairlet decomposition as the
  set of connected components in $G_{E^{\star}}$. To finish the lemma
  we only need to bound the cost of the decomposition.

  For each fairlet, we designate as the center the node of the highest
  degree. It is easy to see that
% for the  $t$-fair $k$-center problem the cost of the solution is bounded by $\tau$. For the 
 %$t$-fair $k$-median
  this solution has cost bounded by $C$.
\end{proof}
\fi
Combining Lemma~\ref{lem:f2m} and Lemma~\ref{lem:m2f} yields the
following.
\begin{lemma}
  By reducing the $(1,t')$-fairlet decomposition problem to an MCF
  problem, it is possible to compute a 2-approximation for the optimal
  $(1,t')$-fairlet decomposition for the $(1/t', k)$-fair center problem.
%$t$ approximation for the optimal  $(1,t)$-fairlet decomposition for the $t$-fair $k$-median problem.
\end{lemma}
Note that the cost of a $(1,t')$-fairlet decomposition is necessarily
smaller than the cost of a $(1/t', k)$-fair clustering. Our main theorem
follows.
\begin{theorem}
\label{thm:tee-fair}
The algorithm that first finds fairlets and then clusters them is a
$4$-approximation for the $(1/t', k)$-fair center problem
%$t+1+\sqrt{3}+\epsilon$ approximation for 
%the $t$-fair $k$-median problem 
for any positive integer $t'$.
\end{theorem}

\subsection{Fair $k$-median}
\label{sec:kmedian}

The results in the previous section can be modified to yield results
for the $(t, k)$-fair median problem with minor changes that we
describe below.

For the perfectly balanced case, as before, we look for a perfect
matching on the bichromatic graph. Unlike, the $k$-center case, 
our goal is to find a perfect matching of minimum total cost,
since that exactly represents the cost of the fairlet decomposition.
Since the best known approximation for $k$-median is $1+
\sqrt{3}+\epsilon$~\citep{LiSvensson13}, we have:
\begin{theorem}
\label{thm:one-fair}
The algorithm that first finds fairlets and then clusters them is 
%a $3$ approximation  for the $1$-fair $k$-center problem and 
a $(2+\sqrt{3}+\epsilon)$-approximation for the $(1, k)$-fair median
problem.
\end{theorem}
To find $(1, t')$-fairlet decompositions for
integral $t' > 1$, we again resort to MCF and create an instance as in
the $k$-center case, but for each $b_i \in B, r_j \in R$, and for each
$1 \le k, \ell \le t$, we set the cost of the edge $(b_i^k, r_j^{\ell})$ to
$d(b_i,r_j)$.  
%
%In a manner similar to the $k$-center fairlet decomposition, let $\Y$
%be a $(1,t')$ fairlet decomposition of cost $C$.  Then it is possible
%to construct a feasible solution of cost $t' C$. 
%
%\begin{lemma}
%  Let $t'$ be a positive integer.  By reducing the $(1,t')$-fairlet
%  decomposition problem to an MCF problem it is possible to compute a
%  $t'$-approximation for the optimal $(1,t')$-fairlet decomposition
%  for the $(\nicefrac{1}{t'},k)$-fair median problem.
%\end{lemma}
%
\begin{theorem}
\label{thm:tee-fair-kc}
The algorithm that first finds fairlets and then clusters them is a
$(t'+1+\sqrt{3}+\epsilon)$-approximation for the
$(\nicefrac{1}{t'},k)$-fair median problem for any positive integer $t'$.
\end{theorem}

\subsection{Hardness}

We complement our algorithmic results with a discussion of computational
hardness for fair clustering.  We show that the question of finding a
good fairlet decomposition is itself computationally hard.  Thus,
ensuring fairness causes hardness, regardless of the underlying
clustering objective.
%We then show (Theorem \ref{thm:tripartite}) that in contrast to the binary case, if the protected attribute can take on three or more different values, finding perfectly balanced fairlets becomes NP-hard.

%We show that the problem of finding an optimal $(1,t)-$fairlet decomposition for the $k$-median problem or an optimal $\nicefrac1t$-fair $k$-median decomposition is NP-hard. Specifically:
\begin{theorem}
\label{thm:hardness}
For each fixed $t' \ge 3$, finding an optimal $(1,t')$-fairlet
decomposition is NP-hard. Also, finding the minimum cost
\emph{$(\nicefrac1{t'}, k)$-fair median} clustering is NP-hard.
\end{theorem}
\iffull
\begin{proof}
  We reduce from the problem of partitioning the node set of a graph
  $G = (V,E)$ into induced subgraphs of order $t'$ each having
  eccentricity $1$. Equivalently, this question asks whether $V$ can be partitioned into pairwise
  disjoint subsets $S_1, \ldots, S_{|V|/t'}$, so that $|S_i| = t'$ and
  $G|S_{t'}$ is a star with $t'-1$ leaves, for each $i = 1, \ldots,
  |V|/t'$. This problem was shown to be NP-hard by \citet{kh78}, see
  also~\cite{vbetal16}. 
%Kirkpatrick and Hell, section "Families of Partitioning Graphs", paragraph D

  Assume that $|V|$ is divisible by $t'$.  We create one red element
  for each node in $V$, and $|V|/t'$ blue elements.  The distance
  between any two red elements will be $1$ if the corresponding nodes
  in $V$ are connected by an edge, and $2$ otherwise. The distance
  between any two blue elements will be $2$.  Finally, the distance
  between any red element and any blue element will be $2$. For the
  fairlet decomposition problem, we ask whether this instance admits a
  $(1,t')$-fairlet decomposition having total cost upper bounded by
  $\left(1 + \frac1{t'}\right) \cdot |V|$. For the $(\nicefrac1{t'},
  k)$-fair median problem, we ask whether the instance admits a
  $k$-clustering, with $k = |V|/t'$, having median cost at most
  $\left(1 + \frac1{t'}\right) \cdot |V|$.

  Observe that the distance function we defined is trivially a metric,
  since all of its values are in $\{1,2\}$.

  Suppose that $G$ can be partitioned into induced subgraphs of order
  $t'$, with node sets $S_1, \ldots, S_{|V|/t'}$, each with
  eccentricity $1$. For each $i = 1, \ldots, |V|/t'$, we create one
  cluster (or, one fairlet) with the red elements corresponding to the
  nodes in $S_i$, and the $i$th blue element. Then, each cluster will
  contain a red element at distance $1$ from each of the other $t'-1$
  red elements, and at distance $2$ from the only blue element. The
  cost of each cluster will then be at most $t'+1$. The total cost is
  then at most $\left(1 + \frac1{t'}\right) \cdot |V|$.

  On the other hand, observe that since the number of blue elements is
  $|V| / t'$ and the number of red elements is $|V|$ any feasible
  solution has to create $k = |V| / t'$ clusters (or fairlets) each
  containing exactly $1$ blue element and $t'$ red elements.  Now, the
  median cost of a cluster (or fairlet) is $t'+1$ if the nodes
  corresponding to its $t'$ red points induce a star, and it is at
  least $t'+2$ otherwise. It follows that, if $G$ cannot be
  partitioned into induced subgraphs of order $t'$ with eccentricity
  $1$, the total cost (of either problems) will be at least
  $\left(1+\frac1{t'}\right) \cdot |V| + 1$.
\end{proof}
\fi

\section{Experiments}

\begin{figure}
    \centering
    \subfloat{{\includegraphics[width=4.7cm]{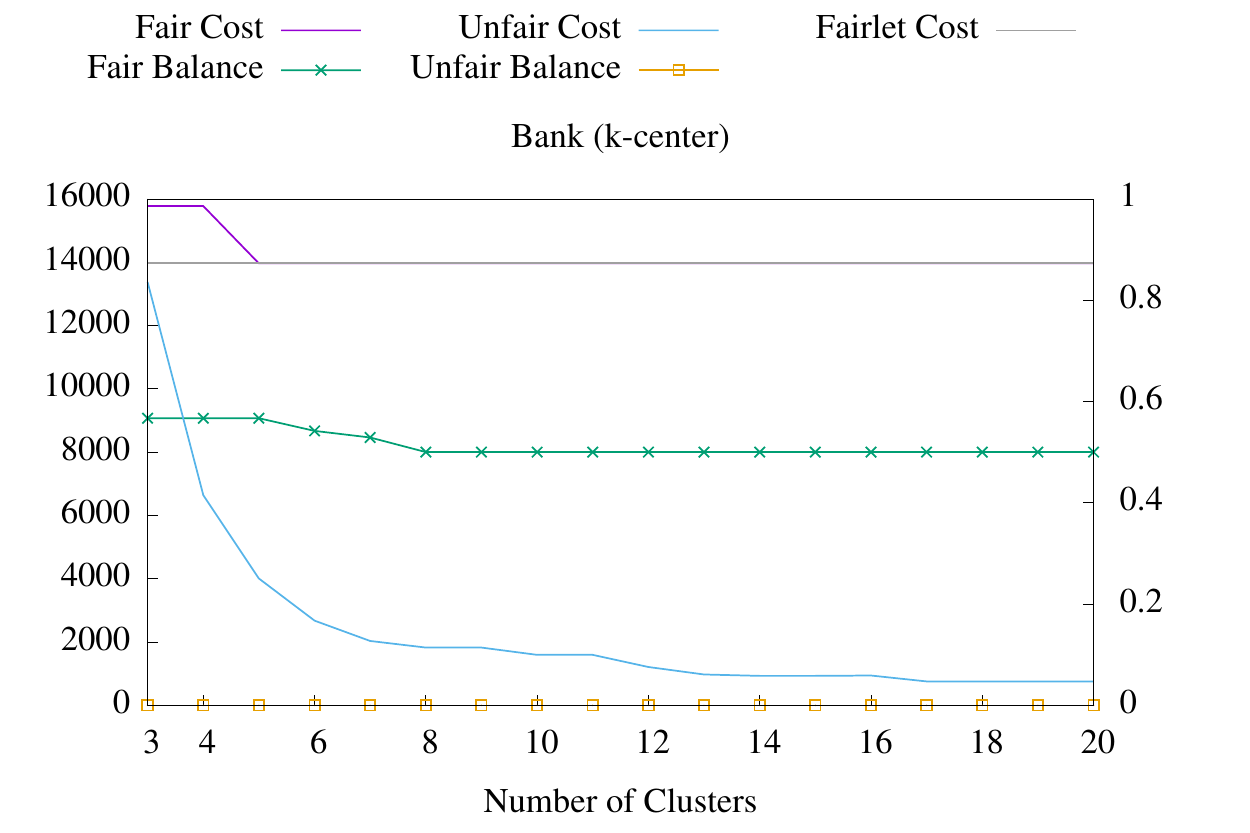} }}%
    \subfloat{{\includegraphics[width=4.7cm]{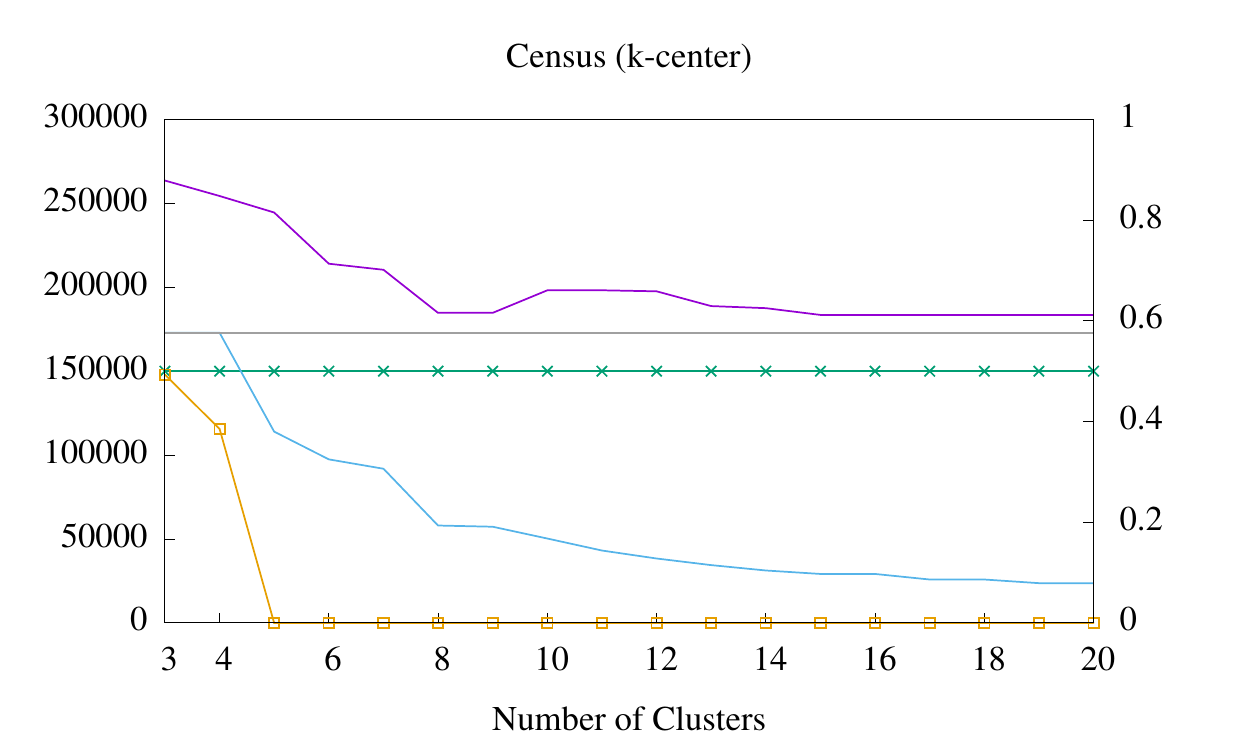} }}%
    \subfloat{{\includegraphics[width=4.7cm]{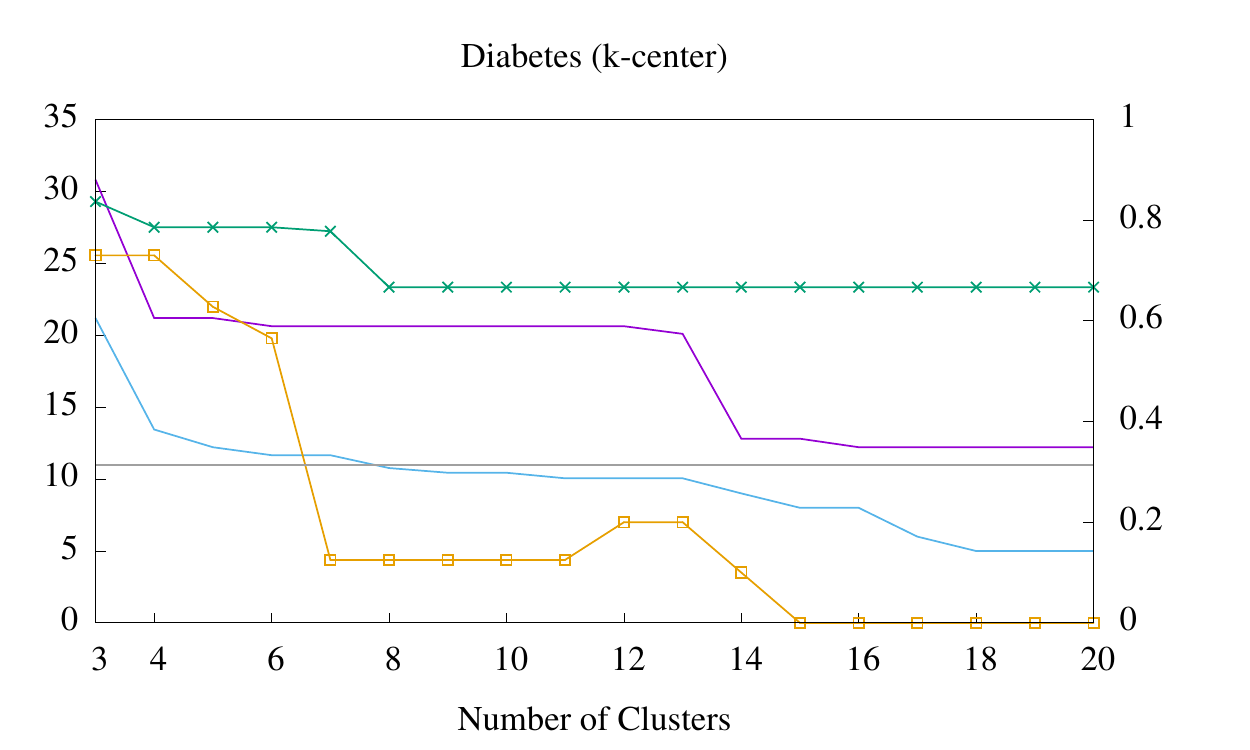} }}%
    \\
    \subfloat{{\includegraphics[width=4.7cm]{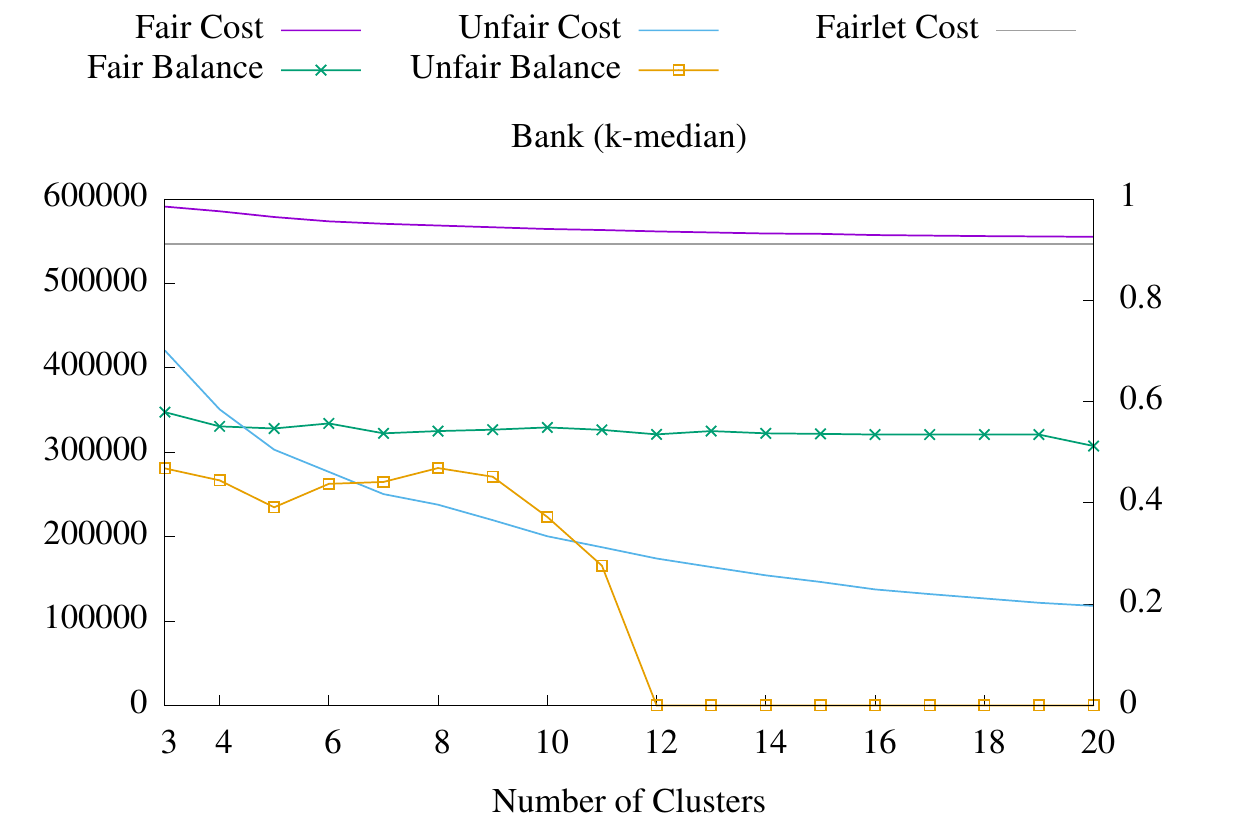} }}%
    \subfloat{{\includegraphics[width=4.7cm]{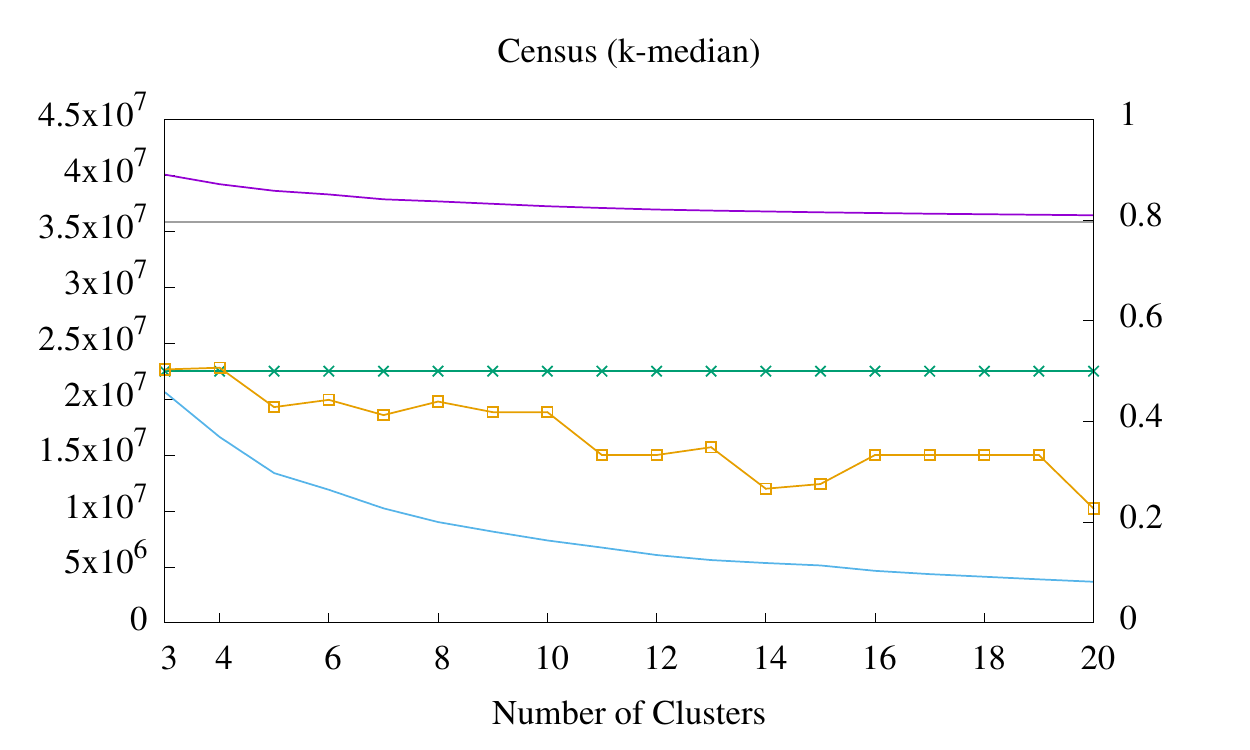} }}%
    \subfloat{{\includegraphics[width=4.7cm]{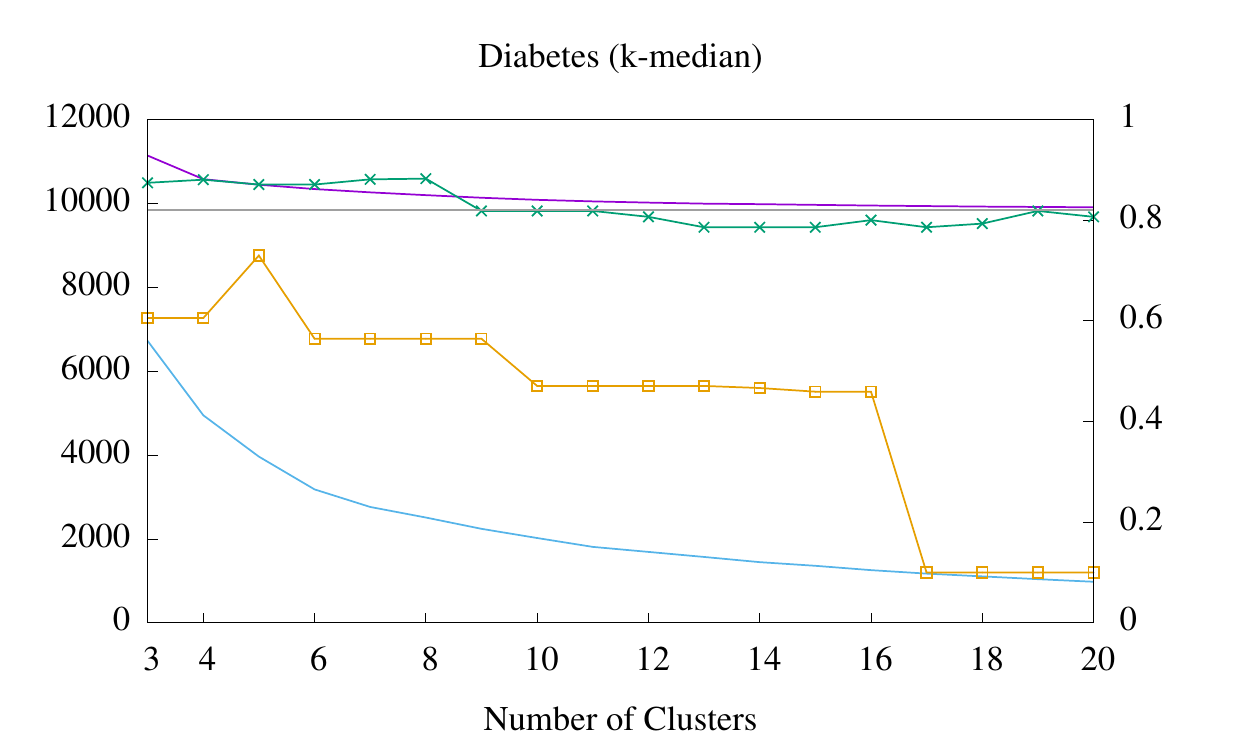} }}%
\caption{Empirical performance of the classical and fair clustering median and center algorithms on the three datasets. The cost of each solution is on left axis, and its balance on the right axis.  }
\label{f:exp}
\end{figure}

In this section we illustrate our algorithm by performing experiments
on real data.  The goal of our experiments is two-fold: first, we show
that traditional algorithms for $k$-center and $k$-median tend to
produce unfair clusters; second, we show that by using our algorithms
one can obtain clusters that respect the fairness guarantees.  We show
that in the latter case, the cost of the solution tends to converge to
the cost of the fairlet decomposition, which serves as a lower bound
on the cost of the optimal solution.

\para{Datasets} 
We consider 3 datasets from the UCI
repository~\cite{Lichman} for experimentation.

{\em Diabetes.} This
dataset\footnote{\url{https://archive.ics.uci.edu/ml/datasets/diabetes}}
represents the outcomes of patients pertaining to diabetes. We chose
numeric attributes such as age, time in hospital, to represent points
in the Euclidean space and gender as the sensitive dimension, i.e., we
aim to balance gender. We subsampled the dataset to 1000
records.
  %This gender attribute in this dataset is 52\% males
%and 48\% females; hence it is reasonably balanced to begin with. We aim for 50\% balance in each cluster. 

{\em Bank.} This
dataset\footnote{\url{https://archive.ics.uci.edu/ml/datasets/Bank+Marketing}}
contains one record for each phone call in a marketing campaign ran by
a Portuguese banking institution~\citep{moro14}.  Each record contains
information about the client that was contacted by the institution. We
chose numeric attributes such as age, balance, and duration to
represent points in the Euclidean space, we aim to cluster to balance
married and not married clients. We subsampled the dataset to 1000
records.
 
{\em Census.} This
dataset\footnote{\url{https://archive.ics.uci.edu/ml/datasets/adult}}
contains the census records extracted from the 1994 US
census~\citep{kohavi}.  Each record contains information about
individuals including education, occupation, hours worked per week,
etc. We chose numeric attributes such as age, fnlwgt, education-num,
capital-gain and hours-per-week to represents points in the Euclidean
space and we aim to cluster the dataset so to balance gender. We
subsampled the dataset to 600 records.

\para{Algorithms} 
We implement the flow-based fairlet decomposition algorithm as
described in Section~\ref{sec:algo}. To solve the $k$-center problem
we augment it with the greedy furthest point algorithm due
to~\cite{Gonzalez85}, which is known to obtain a $2$-approximation. To
solve the $k$-median problem we use the single swap algorithm due
to~\cite{Arya}, which also gets a 5-approximation in the worst case,
but performs much better in practice~\citep{Kanungo2}.

\para{Results}
Figure~\ref{f:exp} shows the results for $k$-center for the three
datasets in the top row and the $k$-median
objective in the bottom row. In all of the cases, we run with $t' = 2$, that is we aim
for balance of at least $0.5$ in each cluster.

Observe that the balance of the solutions produced by the classical
algorithms is very low, and in four out of the six cases, the balance
is $0$ for larger values of $k$, meaning that the optimal solution has
monochromatic clusters.  Moreover, this is not an isolated incident,
for instance the $k$-median instance of the Bank dataset has three
monochromatic clusters starting at $k = 12$. Finally, left unchecked,
the balance in all datasets keeps decreasing as the clustering becomes
more discriminative, with increased $k$.

On the other hand the fair clustering solutions maintain a balanced
solution even as $k$ increases. Not surprisingly, the balance comes
with a corresponding increase in cost, and the fair solutions are
costlier than their unfair counterparts. In each plot we also show the
cost of the fairlet decomposition, which represents the limit of the
cost of the fair clustering; in all of the scenarios the overall cost
of the clustering converges to the cost of the fairlet decomposition.

\section{Conclusions}
\label{sec:conc}

In this work we initiate the study of fair clustering algorithms. Our
main result is a reduction of fair clustering to classical clustering
via the notion of fairlets. We gave efficient approximation algorithms
for finding fairlet decompositions, and proved lower bounds showing
that fairness can introduce a computational bottleneck. An immediate
future direction is to tighten the gap between lower and upper bounds
by improving the approximation ratio of the decomposition algorithms,
or giving stronger hardness results.  A different avenue is to extend
these results to situations where the protected class is not binary,
but can take on multiple values. Here there are multiple challenges
including defining an appropriate version of fairness.

\iffalse
. The first is defining an appropriate
version of fairness---for instance, one can imagine calling a cluster
fair if no single color dominates; a different notion is to insist
that every color is represented approximately equally. A cluster that
has 50\% of points of one color, and 25\% of points of two other
colors is fair under the former definition, but unfair under the
latter. The second difficulty is computational, since both
3-Dimensional Matching, and triangle decomposition of a graph are
NP-hard, there is reason to believe that finding even perfectly
balanced fairlets can no longer be done efficiently.
\fi

\subsubsection*{Acknowledgments}
Flavio Chierichetti was supported in part by the ERC Starting Grant DMAP 680153, by a Google Focused Research Award, and by the SIR Grant RBSI14Q743.

\balance
%\bibliographystyle{authordate1}
%\bibliography{paper}

\end{document}